\documentclass[letterpaper]{article} 
\usepackage{aaai24}  
\usepackage{times}  
\usepackage{helvet}  
\usepackage{courier}  
\usepackage[hyphens]{url}  
\usepackage{graphicx} 
\usepackage{natbib}  
\usepackage{caption} 
\usepackage{newfloat}
\usepackage{listings}
\DeclareCaptionStyle{ruled}{labelfont=normalfont,labelsep=colon,strut=off} 
\usepackage[linesnumbered,vlined,ruled]{algorithm2e}
\SetCommentSty{emph}
\SetKwProg{Fn}{\textbf{Function}}{}{}
\let\oldnl\nl
\newcommand{\nonl}{\renewcommand{\nl}{\let\nl\oldnl}}

\usepackage{enumitem}

\usepackage[dvipsnames]{xcolor}
\usepackage{subcaption}
\usepackage{amsthm}

\theoremstyle{definition}
\newtheorem{definition}{Definition}

\newtheorem{assumption}{Assumption}

\newtheorem{construction}{Construction}
\newtheorem{example}{Example}

\theoremstyle{remark}
\newtheorem{remark}{Remark}

\theoremstyle{plain}
\newtheorem{thm}{Theorem}
\newtheorem{lem}[thm]{Lemma}
\newtheorem{prop}[thm]{Proposition}
\newtheorem{cor}[thm]{Corollary}

\usepackage{amsmath,amssymb}

\usepackage{cleveref}
\title{A Real-Time Rescheduling Algorithm for Multi-robot Plan Execution}
\author {
    Ying Feng\textsuperscript{\rm 1},
    Adittyo Paul\textsuperscript{\rm 1},
    Zhe Chen\textsuperscript{\rm 2},
    Jiaoyang Li\textsuperscript{\rm 1}
}
\affiliations {
    \textsuperscript{\rm 1}Carnegie Mellon University, USA\\
    \textsuperscript{\rm 2}Monash University, Australia\\
    \{yingfeng,adittyop\}@andrew.cmu.edu, zhe.chen@monash.edu, jiaoyangli@cmu.edu
}

\usepackage{bibentry}

\begin{document}

\maketitle

\begin{abstract}
One area of research in multi-agent path finding is to determine how replanning can be efficiently achieved in the case of agents being delayed during execution. One option is to reschedule the passing order of agents, i.e., the sequence in which agents visit the same location. In response, we propose Switchable-Edge Search (SES), an A*-style algorithm designed to find optimal passing orders. We prove the optimality of SES and evaluate its efficiency via simulations. The best variant of SES takes less than 1 second for small- and medium-sized problems and runs up to 4 times faster than baselines for large-sized problems.
\end{abstract}

\section{Introduction}

Multi-Agent Path Finding (MAPF) is the problem of finding collision-free paths that move a team of agents from their start to goal locations. 
MAPF is fundamental to numerous applications,
such as automated warehouses~\cite{wurman2008coordinating,meng2019idle},  computer games~\cite{silver2005cooperative}, and drone swarms~\cite{honig2018large}.

Classic MAPF models assume flawless execution. 
However, in real-world scenarios, agents may encounter unexpected delays due to mechanical differences, unforeseen events, localization errors, and so on. 
To accommodate such delays, existing research suggests the use of a Temporal Plan Graph (TPG) \cite{honig2016}. 
The TPG captures the precedence relationships within a MAPF solution and maintains them during execution. Each precedence relationship specifies an order for two agents to visit the same location. 
An agent advances to the next location in its path only if the corresponding precedence conditions are met. 
Consequently, if an agent experiences a delay, all other agents whose actions depend on this agent will pause. 
Despite its advantages of providing rigorous guarantees on collision-freeness and deadlock-freeness, the use of TPG can introduce a significant number of waits into the execution results due to the knock-on effect in the precedence relationship.

In this paper, we adopt a variant of TPG, named \emph{Switchable TPG} (STPG) \cite{Berndt2020AFS}. 
STPG allows for the modification of some precedence relationships in a TPG, resulting in new TPGs. 
To address delays, we propose an A*-style algorithm called \emph{Switchable-Edge Search} (SES) to find the new TPG based on a given STPG that minimizes the travel times of all agents to reach their goal locations. 
We prove the optimality of SES and introduce two variants: an execution-based variant (ESES) and a graph-based variant (GSES).
Experimental results show that GSES finds the optimal TPG with an average runtime of less than 1 second for various numbers of agents on small- and medium-sized maps. On larger maps, GSES runs up to 4 times faster than existing replanning algorithms.

\section{Preliminaries}


\begin{definition}[MAPF]\label{definition:MAPF}
Multi-Agent Path Finding (MAPF) aims to find collision-free paths for a team of agents $\mathcal{A}$ on a given graph. Each agent $i \in \mathcal{A}$ has a unique start location and a unique goal location. In each discrete timestep, every agent either moves to an adjacent location or waits at its current location. A path for an agent specifies its action at each timestep from its start to goal locations. A collision occurs if either of the following happens:%
\begin{enumerate}
    \item Two agents are at the same location at the same timestep.

    \item One agent leaves a location at the same timestep when another agent enters the same location.%
\end{enumerate}%
A MAPF solution is a set of collision-free paths of all agents.%
\end{definition}%

\begin{remark}
 The above definition of collision coincides with that in the setting of k-robust MAPF \cite{k-robust} with $k=1$. We disallow the second type of collision because, if agents follow each other and the front agent suddenly stops, the following agents may collide with the front agent. Thus, this restriction ensures better robustness when agents are subject to delays. Note that the swapping collision, where two agents swap their locations simultaneously, is a special case of the second type of collision. 
\end{remark}

A MAPF solution can be represented in different formats. We stick to the following format for our discussion, though our algorithms do not depend on specific formats.

\begin{definition}[MAPF Solution]
\label{assumption:MAPF}
    A MAPF solution takes the form of a set of collision-free paths $\mathcal{P} = \{p_i: i \in \mathcal{A}\}$. Each path $p_i$ is a sequence of location-timestep tuples $(l^i_0, t^i_0) \to (l^i_1, t^i_1) \to \cdots \to (l^i_{zi}, t^i_{zi})$ with the following properties:
(1) The sequence follows a strict temporal ordering: $0 = t^i_0 < t^i_1 < \cdots < t^i_{zi}$.
(2) $l^i_0$ and $l^i_{zi}$ are the start and goal locations of agent $i$, respectively. 
(3) Each tuple $(l^i_k, t^i_k)$ with $k > 0$ specifies a move action of $i$ from $l^i_{k-1}$ to $l^i_k$ at timestep $t^i_k$.

These properties force all consecutive pairs of locations $l^i_{k}$ and $l^i_{k+1}$ to be adjacent on the graph. A wait action is implicitly defined between two consecutive tuples. Namely, if $ t^i_{k+1} - t^i_k= \Delta > 1$, then $i$ is planned to wait at $l^i_k$ for $\Delta-1$ timesteps before moving to $l^i_{k+1}$. Additionally, $t^i_{zi}$ records the time when $i$ reaches its goal location, called the \emph{travel time} of $i$. The cost of a MAPF solution $\mathcal{P}$ is $cost(\mathcal{P})=\sum_{i \in \mathcal{A}} t^i_{zi}$, and $\mathcal{P}$ is \emph{optimal} if its cost is minimum.
\end{definition}

\begin{remark}
    \Cref{assumption:MAPF} discards the explicit representation of wait actions. This is because, when executing $\mathcal{P}$ as a TPG (specified in the next section), it may reduce the travel time if $\mathcal{P}$ has unnecessary wait actions, e.g., when $\mathcal{P}$ is suboptimal.
\end{remark}




\paragraph{Related Works}
Numerous recent studies on MAPF have explored strategies for managing unexpected delays during execution.
A simple strategy is to re-solve the MAPF problem when a delay occurs. 
However, this strategy is computationally intensive, leading to prolonged agent waiting time. 
To avoid the need for replanning, \citet{k-robust} suggested the creation of a $k$-robust MAPF solution, allowing agents to adhere to their planned paths even if each agent is delayed by up to $k$ timesteps. 
However, replanning is still required if an agent's delay exceeds $k$ timesteps.
\citet{atzmon2020probabilistic} then proposed a different model, called $p$-robust MAPF solutions, that ensures execution success with a probability of at least $p$, given an agent delay probability model. 
Nevertheless, planning a $k$-robust or $p$-robust MAPF solution is considerably more computational-intensive than computing a standard MAPF solution.
Another strategy for managing delays involves the use of an execution policy that preserves the precedence relationships of a MAPF solution during execution \cite{honig2016, ma2017multi, honig2019persistent}. 
This strategy is quick and eliminates the need for replanning paths. 
However, the execution results often leave room for improvement, 
as many unnecessary waits are introduced. 
Our work aims to address this limitation by formally exploring the concept of optimizing precedence relationships online~\cite{Berndt2020AFS, MannucciTRO21}.

\section{Temporal Plan Graph (TPG)}\label{sec:TPG}

In essence, we aim to optimize the passing order for multiple agents to visit the same location. This is achieved using a graph-based abstraction known as the TPG.

\begin{definition}[TPG]\label{definition:TPG}
A Temporal Plan Graph (TPG)~\cite{honig2016} is a directed graph $\mathcal{G} = (\mathcal{V}, \mathcal{E}_1, \mathcal{E}_2)$ that represents the precedence relationships of a MAPF solution $\mathcal{P}$. The set of vertices is $\mathcal{V} = \{v^i_k: i\in \mathcal{A}, k \in [0, zi]\}$, where each vertex $v^i_k$ corresponds to $(l_k^i, t_k^i)$, namely the $k^\text{th}$ move action in path $p_i$. There are two types of edges $\mathcal{E}_1$ and $\mathcal{E}_2$, where each directed edge $(u,v) \in \mathcal{E}_1 \cup \mathcal{E}_2$ encodes a precedence relationship between a pair of move actions, namely movement $u$ is planned to happen \emph{before} movement $v$.
\begin{itemize}
    \item A \emph{Type 1 edge} connects two vertices of the same agent, specifying its path. Specifically, $\mathcal{E}_1 = \{(v_{k}^i, v_{k+1}^i): \forall i \in \mathcal{A}, k \in [0, zi) \}$.
    \item A \emph{Type 2 edge} connects two vertices of distinct agents, specifying their ordering of visiting the same location. Specifically, 
        $\mathcal{E}_2 = \{(v_{s+1}^j, v_{k}^i):  \forall i\neq j \in \mathcal{A}, s \in [0, zj), k \in [0, zi] 
        \text{ satisfying } l_{s}^j =l_{k}^i \text{ and } t_{s+1}^j < t_{k}^i\}$. 
\end{itemize}
\end{definition}%
\begin{figure}
    \centering
    \includegraphics[width=\linewidth]{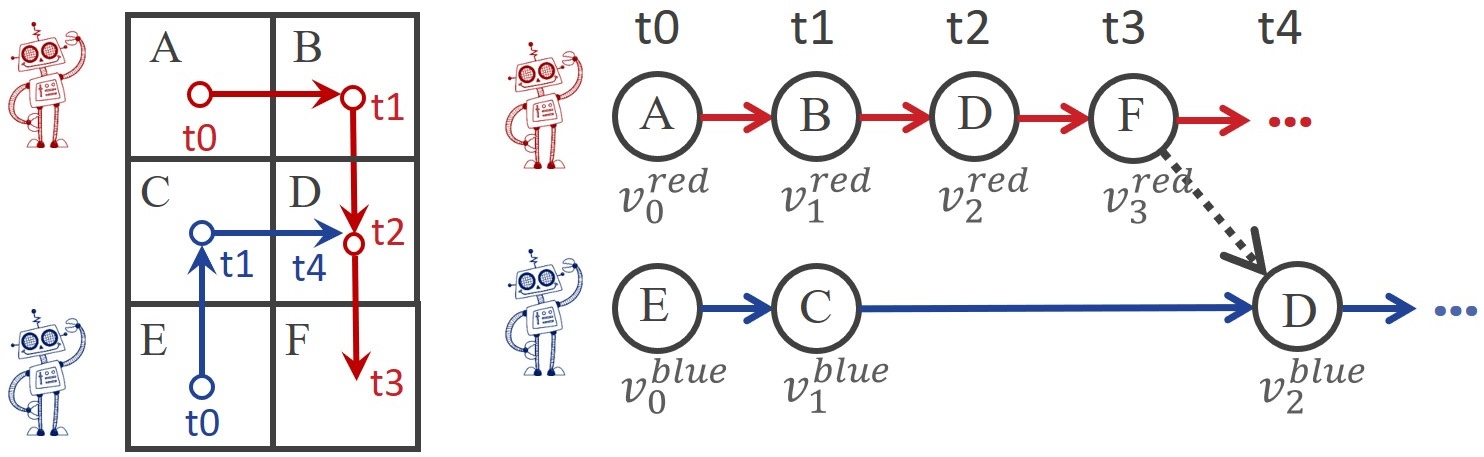}%
    \caption{Example of converting a MAPF solution to a TPG. The solid arrows in the TPG represent Type 1 edges, and the dashed arrow represents a Type 2 edge.}
    \label{figure:example1}
\end{figure}
\begin{example} \Cref{figure:example1} shows an example of converting a MAPF solution into a TPG. Both agents are planned to visit location D, and the red agent is planned to visit D earlier than the blue agent. Consequently, there is a Type 2 edge from $v^{\text{red}}_3$ to $v^{\text{blue}}_2$, signifying that the blue agent can move to D only after the red agent has reached F. Note that we define Type 2 edge as $(v^{\text{red}}_3, v^{\text{blue}}_2)$ instead of $(v^{\text{red}}_2, v^{\text{blue}}_2)$ to avoid the second type of collision in Definition \ref{definition:MAPF}.
\end{example}

\begin{algorithm}[t!]
 \SetAlgorithmName{Procedure}{algorithmautorefname}{list of algorithms name}
 \caption{TPG Execution\\
 Codes highlighted in \textcolor{blue}{blue} compute $cost(\mathcal{G})$ and can be omitted for the mere purpose of execution.}
 \label{procedure:execute}
 
\textcolor{blue}{Define a counter $cost$\;}
\Fn{\textsc{Init$_{\textsc{exec}}$}$(\mathcal{V})$}
{
    \textcolor{blue}{$cost \gets 0$\;}
    Mark vertices in $\mathcal{V}_0 = \{v_0^i : i \in \mathcal{A}\}$ as satisfied\;\label{line:exec:init:mark-satisfied}
    Mark vertices in $(\mathcal{V} \setminus \mathcal{V}_0)$ as unsatisfied\;\label{line:exec:init:mark-unsatisfied}
}
\Fn{\textsc{Step$_{\textsc{exec}}$}$(\mathcal{G}=(\mathcal{V}, \mathcal{E}_1, \mathcal{E}_2))$}
{
    $\mathcal{S} \gets \emptyset$\tcp*{Vertices to be marked as satisfied}
    \ForAll{$i \in \mathcal{A} : v^i_{zi}$ unsatisfied}{
        \textcolor{blue}{$cost \gets cost+1$\;}
        $k \gets \min\{k: v^i_{k} \text{ unsatisfied}\}$\;
        \If{$\forall (u, v_k^i) \in \mathcal{E}_2 : u$ \normalfont{satisfied}}{Add $v^i_{k}$ into $\mathcal{S}$\;\label{line:exec:v-exists}}
    }

    \Return $\mathcal{S}$\;\label{line:exec:step:return}
}
\Fn{\textsc{Exec$(\mathcal{G}=(\mathcal{V}, \mathcal{E}_1, \mathcal{E}_2))$}}
{   
    \textsc{Init$_{\textsc{EXEC}}$}($\mathcal{V}$)\;
    \While{$\exists v \in \mathcal{V} : v$ \normalfont{unsatisfied}\label{line:exec-while}}{
    $S \gets$ \textsc{Step$_{\textsc{exec}}$}$(\mathcal{G})$\;\label{line:exec:call-step}
    \lForAll{$v^i_{k} \in \mathcal{S}$\label{line:exec:for-s}}{Mark $v^i_{k}$ as satisfied\label{line:exec:mark-satisfied}}}
\textcolor{blue}{\Return{$cost$}\;}\label{line:exec:return}
}
\end{algorithm}
\paragraph{Executing a TPG}
Procedure \ref{procedure:execute} describes how to execute a TPG $\mathcal{G}$, which includes a main function $\textsc{exec}$ and two helper functions $\textsc{init}_\textsc{exec}$ and $\textsc{step}_\textsc{exec}$, along with two marks ``satisfied'' and ``unsatisfied'' for vertices. Marking a vertex as satisfied corresponds to moving an agent to the corresponding location, and we do so if and only if all in-neighbors of this vertex have been satisfied. The execution terminates when all vertices are satisfied, i.e., all agents have reached their goal locations. The \emph{cost} of executing $\mathcal{G}$, namely $cost(\mathcal{G}) = \textsc{Exec}(\mathcal{G})$, is the sum of travel time for agents following $\mathcal{G}$ (while assuming no delays happen).

We now introduce some known properties of TPGs. All proofs are delayed to the appendix. We use $\mathcal{G}$ to denote a TPG constructed from a MAPF solution $\mathcal{P}$ as in \Cref{definition:TPG}.


\begin{prop}[Cost]\label{proposition:cost}
    $cost(\mathcal{G}) \le cost(\mathcal{P})$.
\end{prop}

Intuitively, $cost(\mathcal{G}) < cost(\mathcal{P})$ if $\mathcal{P}$ has unnecessary wait actions and $cost(\mathcal{G}) = cost(\mathcal{P})$ otherwise.

\begin{prop}[Collision-Free]\label{proposition:collision}
    Executing a TPG with Procedure~\ref{procedure:execute} does not lead to collisions among agents.
\end{prop}

Next, we present two lemmas regarding {\it deadlocks} of executing a TPG, which were used in previous work~\cite{Berndt2020AFS,SuAAAI24} and are helpful for our discussion of switchable TPGs in the next section.

\begin{definition}[Deadlock]\label{definition:deadlock}
    When executing a TPG, a \emph{deadlock} is encountered iff, in an iteration of the while-loop of \textsc{Exec}($\mathcal{G}$), $\mathcal{V}$ contains unsatisfied vertices but $\mathcal{S} = \emptyset$. 
\end{definition}

\begin{lem}[Deadlock $\iff$ Cycle]\label{lemma:cycle}
    Executing a TPG encounters a deadlock iff the TPG contains cycles.
\end{lem}

\begin{lem}[Deadlock-Free]\label{corollary:deadlock-free}    
    If a TPG is constructed from a MAPF solution, then executing it is deadlock-free.
\end{lem}

\section{Switchable TPG (STPG)} 
TPG is a handy representation for precedence relationships. Yet, a TPG constructed as in Definition \ref{definition:TPG} is fixed and bound to a given set of paths. In contrast, our optimization algorithm will use the following extended notion of TPGs, which enables flexible modifications of precedence relationships.

\begin{definition}[STPG]\label{definition:switchableTPG}
Given a TPG $\mathcal{G} = (\mathcal{V}, \mathcal{E}_1, \mathcal{E}_2)$, a \emph{Switchable TPG} (STPG) $\mathcal{G^S} = (\mathcal{V}, \mathcal{E}_1, (\mathcal{S}_{\mathcal{E}2}, \mathcal{N}_{\mathcal{E}2}))$ partitions Type 2 edges $\mathcal{E}_2$ into two disjoint subsets $\mathcal{S}_{\mathcal{E}2}$ (switchable edges) and $\mathcal{N}_{\mathcal{E}2}$ (non-switchable edges) and allows two operations on any switchable edge $(v_{s+1}^j, v_{k}^i) \in \mathcal{S}_{\mathcal{E}2}$:%
\begin{itemize}
    \item $fix(v_{s+1}^j, v_{k}^i)$ removes $(v_{s+1}^j, v_{k}^i)$ from $\mathcal{S}_{\mathcal{E}2}$ and add it into $\mathcal{N}_{\mathcal{E}2}$. It fixes a switchable edge to be non-switchable.
    \item $reverse(v_{s+1}^j, v_{k}^i)$ removes $(v_{s+1}^j, v_{k}^i)$ from $\mathcal{S}_{\mathcal{E}2}$ and add $(v_{k+1}^i, v_{s}^j)$ into $\mathcal{N}_{\mathcal{E}2}$. It switches the precedence relationship and then fixes it to be non-switchable.
\end{itemize}
\end{definition}

\begin{remark}\label{remark:switch}
    Reversing the precedence relationship represented by $(v_{s+1}^j, v_{k}^i)$ produces  $(v_{k+1}^i, v_{s}^j)$ because, based on \Cref{definition:TPG}, Type 2 edge $(v_{s+1}^j, v_{k}^i)$ indicates locations $l^j_s$ and $l^i_k$ are the same. Thus, after reversing, vertex $v^i_{k+1}$ 
    needs to be satisfied before $v_{s}^j$ can be marked as satisfied.
\end{remark}

\begin{example}
\Cref{figure:example2} shows an example of reversing an edge. After the $reverse$ operation, edge $(v^{\text{red}}_1, v^{\text{blue}}_2)$ in the left TPG is replaced with edge $(v^{\text{blue}}_3, v^{\text{red}}_0)$ in the right TPG. 
\end{example}


\begin{figure}
    \centering
    \includegraphics[width=\linewidth]{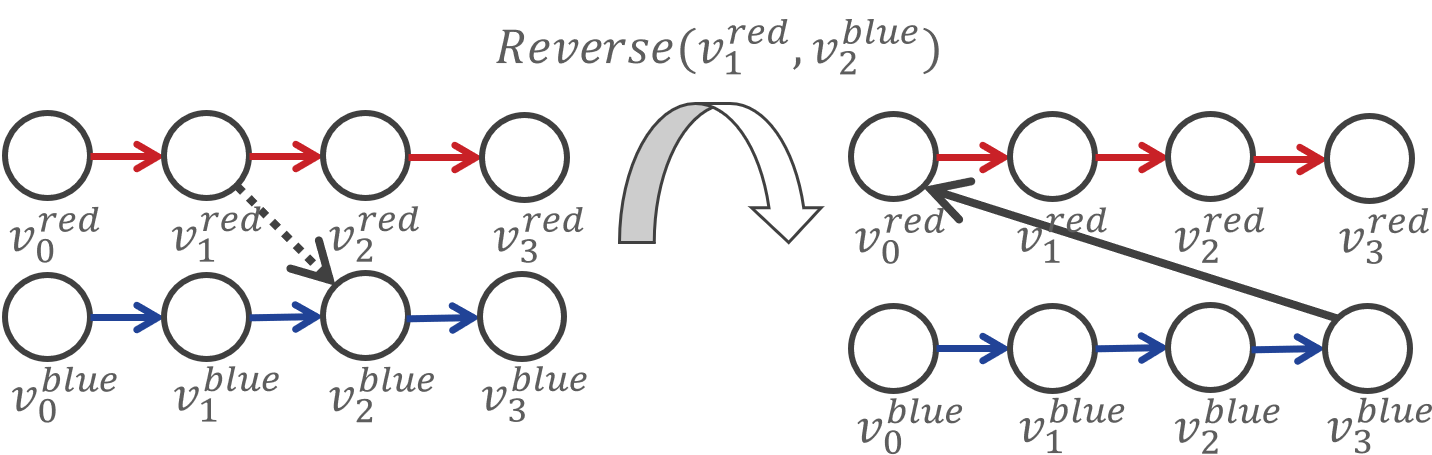}
    \caption{Example of reversing an edge in a TPG.}
    \label{figure:example2}
\end{figure}

\Cref{definition:switchableTPG} defines a strict superclass of \Cref{definition:TPG}. 
A STPG degenerates into a TPG if $\mathcal{S}_{\mathcal{E}2}$ is empty. 

\begin{definition}[$\mathcal{G^S}$-producible TPG]\label{definition:produceTPG} 
Given a STPG $\mathcal{G^S}$, a TPG is $\mathcal{G^S}$-producible if it can be generated through a sequence of $fix$ and $reverse$ operations on $\mathcal{G^S}$.
\end{definition}


We now show the roadmap of our algorithm.
Given a MAPF solution $\mathcal{P}$, we construct TPG $\mathcal{G}_0$ from $\mathcal{P}$ as in \Cref{definition:TPG} and then run Procedure \ref{procedure:execute}. When a delay happens, we (1) construct a STPG $\mathcal{G^S}$ based on $\mathcal{G}_0$ and 
(2) finds a TPG $\mathcal{G}^*$ with $cost(\mathcal{G}^*)= \min\{cost(\mathcal{G}) : \mathcal{G}$ is $\mathcal{G^S}$-producible$\}$, representing an optimal ordering of agents visiting each location, upon sticking to the original location-wise paths. We describe Step (1) below and Step (2) in the next section. 

\begin{construction}\label{construction}
     Assume that, during the execution of $\mathcal{G}_0=(\mathcal{V}, \mathcal{E}_1, \mathcal{E}_2)$, agent $d$ is forced to delay at its current location $l_{c}^d$ for $\Delta$ timesteps. 
     We construct STPG $\mathcal{G^S}$ as follows:
    \begin{enumerate}
        \item Construct STPG $\mathcal{G^S}=(\mathcal{V}, \mathcal{E}_1, (\mathcal{S}_{\mathcal{E}2}, \mathcal{N}_{\mathcal{E}2}))$ with $\mathcal{S}_{\mathcal{E}2} = \{(v_{s+1}^j, v_k^i) \in \mathcal{E}_2 : $ $v_{s+1}^j$ is unsatisfied and $k < zi\}$ and $\mathcal{N}_{\mathcal{E}2} = \{(v_{s+1}^j, v_k^i) \in \mathcal{E}_2 : $ $v_{s+1}^j$ is satisfied or $k=zi\}$.

        \item Create $\Delta$ new dummy vertices $\mathcal{V}_{\text{new}} = \{v_1, \cdots, v_\Delta \}$ and $(\Delta+1)$ new Type-1 edges $\mathcal{E}_{\text{new}} =\{ (v^d_{c}, v_1),$ $(v_1, v_2), \cdots, (v_{\Delta-1}, v_{\Delta}), (v_{\Delta}, v^d_{c+1})\}$ and modify $\mathcal{G^S}$ with $\mathcal{V} \gets \mathcal{V} \cup \mathcal{V}_{\text{new}}$  and 
        $\mathcal{E}_1 \gets (\mathcal{E}_1 \cup \mathcal{E}_{\text{new}}) \setminus \{(v^d_{c}, v^c_{c+1})\}$.
    \end{enumerate}
    If there are multiple agents delayed at the same timestep, we repeat Step 2 for each delayed agent.
\end{construction}

\begin{remark}
    In Step 1, $(v_{s+1}^j, v_k^i)$ is non-switchable when $v_{s+1}^j$ is satisfied because agent $j$ has already visited $l_{s}^j$.  $(v_{s+1}^j, v_{zi}^i)$ is non-switchable because agent $i$ must be the last one to visit its goal location. 
    The dummy vertices added in Step 2 are used to account for the delays in Procedure~\ref{procedure:execute}.
\end{remark}

We now show an intuitive yet crucial theorem.

\begin{thm}\label{theorem:existence}
    If STPG $\mathcal{G^S}$ is constructed by Construction $\ref{construction}$, then there is at least one deadlock-free $\mathcal{G^S}$-producible TPG.
\end{thm}

\begin{proof}
     We generate a na\"ive solution $\mathcal{G}_{\text{naive}}$ by $fix$ing all switchable edges in $\mathcal{G^S}$. Lemma \ref{corollary:deadlock-free} ensures that $\mathcal{G}_0$ is deadlock-free. $\mathcal{G^S}$ constructed in Step 1 is identical to $\mathcal{G}_0$ if we $fix$ all switchable edges. Step 2 behaves as expanding a pre-existing edge $(v^i_{k-1}, v^i_k)$ into a line of connecting edges, which does not create any new cycles. Therefore, by Lemma \ref{lemma:cycle}, $\mathcal{G}_{\text{naive}}$ is deadlock-free.
\end{proof}

\section{Switchable Edge Search (SES) Framework}

\begin{algorithm}[t!]
 \SetAlgorithmName{Algorithm}{algorithmautorefname}{list of algorithms name}
 \caption{Switchable Edge Search (SES) \\ \textsc{Heuristic} and \textsc{Branch} are modules to be specified later. $\mathcal{X}$ stores auxiliary information accompanying a STPG and will be specified later. }
 \label{algorithm:framework}

\KwIn{STPG $\mathcal{G}^\mathcal{S}_{\text{root}}$}%
\KwOut{TPG $\mathcal{G}$}


$(h_{\text{root}}, \mathcal{X}_{init}) \gets \textsc{heuristic}(\mathcal{G}^\mathcal{S}_{\text{root}}, \mathcal{X}_{init})$\;\label{line:SES:root-h}

$\mathcal{Q}\gets \{(\mathcal{G}^\mathcal{S}_{\text{root}}, \mathcal{X}_{init}, 0, h_{\text{root}})\}$\tcp*{A priority queue}

\While{$\mathcal{Q}$ \normalfont{is not empty}\label{line:SES:while}}{
$(\mathcal{G}^\mathcal{S}, \mathcal{X}, g, h) \gets \mathcal{Q}.pop()$\;\label{line:SES:pop}


$(\mathcal{X}', g_\Delta, (v_{k+1}^i, v_{s}^j)) \gets \textsc{Branch}(\mathcal{G}^\mathcal{S}, \mathcal{X})$\;\label{line:SES:call-branch}

\lIf{$(v_{k+1}^i, v_{s}^j) = \textsc{Null}$}{\Return $\mathcal{G}^\mathcal{S}$}


$\mathcal{G}^\mathcal{S}_{\text{f}} \gets fix(\mathcal{G}^\mathcal{S}, (v_{k+1}^i, v_{s}^j))$\;\label{line:SES:fix}

\If{\normalfont{\textbf{not} \textsc{CycleDetection}($\mathcal{G}^\mathcal{S}_{\text{f}}$, $(v_{k+1}^i, v_{s}^j)$)}\label{line:SES:detect-cycle1}}{
    $(h_{f}, \mathcal{X}_{f}) \gets \textsc{Heuristic}(\mathcal{G}^\mathcal{S}_{\text{f}}, \mathcal{X}')$\;
    $\mathcal{Q}.push((\mathcal{G}^\mathcal{S}_{\text{f}},  \mathcal{X}_{f}, g+g_\Delta, h_{f}))$\;\label{line:SES:insert1}
}

$\mathcal{G}^\mathcal{S}_{\text{r}} \gets reverse(\mathcal{G}^\mathcal{S}, (v_{k+1}^i, v_{s}^j))$\;\label{line:SES:reverse}

\If{\normalfont{\textbf{not} \textsc{CycleDetection}($\mathcal{G}^\mathcal{S}_{\text{r}}$, $(v_{s+1}^j, v_{k}^i)$)}\label{line:SES:detect-cycle2}}{
$(h_{r}, \mathcal{X}_{r}) \gets \textsc{Heuristic}(\mathcal{G}^\mathcal{S}_{\text{r}},\mathcal{X}')$\;
$\mathcal{Q}.push((\mathcal{G}^\mathcal{S}_{\text{r}}, \mathcal{X}_{r}, g+g_\Delta, h_{r}))$\;\label{line:SES:insert2}}

}

    \textbf{throw exception} ``No solution found''\;\label{line:exception}


    \Fn{\textsc{CycleDetection}$(\mathcal{G}^\mathcal{S}, (u, v) )$}
    {
        Run depth-first search (DFS) from $v$ on $red(\mathcal{G}^\mathcal{S})$\; 
        \lIf{\normalfont{DFS visits vertex $u$}}{\Return true}
        \Return false\;
    }
    
\end{algorithm}

We describe our algorithm, \emph{Switchable Edge Search} (SES), in a top-down modular manner, starting with a high-level heuristic search framework in Algorithm \ref{algorithm:framework}. 
We define the \emph{partial cost} of a STPG as the cost of its \emph{reduced} TPG, which is defined as follows.

\begin{definition}[Reduced TPG]\label{definition:reduceTPG}
The reduced TPG of a STPG $\mathcal{G^S} = (\mathcal{V}, \mathcal{E}_1, (\mathcal{S}_{\mathcal{E}2}, \mathcal{N}_{\mathcal{E}2}))$ is the TPG that omits all switchable edges, denoted as $red(\mathcal{G^S}) = (\mathcal{V}, \mathcal{E}_1, \mathcal{N}_{\mathcal{E}2})$. 
\end{definition}

\begin{lem}\label{lemma:partial}
    The partial cost of a STPG $\mathcal{G^S}$ is no greater than the cost of any $\mathcal{G^S}$-producible TPG.
\end{lem}
\begin{proof}
    Let $\mathcal{G}$ be a $\mathcal{G}^\mathcal{S}$-producible TPG. Consider running Procedure~\ref{procedure:execute} on $\mathcal{G}$ and $red(\mathcal{G^S})$, respectively. Since an edge appears in $red(\mathcal{G^S})$ must appear in $\mathcal{G}$, we can inductively show that, in any call to $\textsc{step}_\textsc{exec}$, if a vertex $v$ can be marked as satisfied in $\mathcal{G}$, then it can be marked as satisfied in $red(\mathcal{G^S})$. Thus, the total timesteps to satisfy all vertices in $red(\mathcal{G^S})$ cannot exceed that in $\mathcal{G}$.
\end{proof}

As shown in \Cref{algorithm:framework}, SES runs A* in the space of STPGs with a root node corresponding to the STPG $\mathcal{G}^\mathcal{S}_{\text{root}}$ constructed as in Construction \ref{construction}. The priority queue $\mathcal{Q}$ sorts its nodes by their $f$-values (namely $g+h$). The $f$-value of a node is defined as the partial cost of its STPG. When expanding a node, SES selects one switchable edge in the STPG by module $\textsc{Branch}$ and generates two child nodes with the selected edge being $fix$ed or $reverse$d. We abuse the operators $fix$ and $reverse$ on \Cref{line:SES:fix,line:SES:reverse} to take a STPG and a switchable edge as input and return a new STPG. 

SES uses function \textsc{CycleDetection} to prune child nodes with STPGs that definitely produce cyclic TPGs, namely STPGs whose reduced TPGs are cyclic. Specifically, \textsc{CycleDetection}($\mathcal{G}^\mathcal{S}$, $(u, v))$ returns true iff $red(\mathcal{G}^\mathcal{S})$ contains a cycle involving edge $(u,v)$. As $\mathcal{G}^\mathcal{S}_\text{root}$ is acyclic, it holds inductively that \textsc{CycleDetection}($\mathcal{G}$, $(u, v))$ returns true iff $red(\mathcal{G}^\mathcal{S})$ contains \emph{any} cycle. This is because, when we generate a node, we add only one new non-switchable edge, so any cycle formed must contain the new edge.



\begin{assumption}\label{assumption}
    The modules in SES satisfy:

\begin{enumerate}[label={A\arabic*}]
    \item \textsc{Branch}$(\mathcal{G}^\mathcal{S}, \mathcal{X})$ outputs an updated auxiliary information $\mathcal{X}'$, a value $g_\Delta$, and a switcable edge of $\mathcal{G}^\mathcal{S}$ if exists or $\textsc{Null}$ otherwise.\label{assumption:branch}
    \item \textsc{Heuristic}$(\mathcal{G}^\mathcal{S}, \mathcal{X})$ computes a value $h$ such that $g+ h$ is the partial cost of $\mathcal{G}^\mathcal{S}$ for every node $(\mathcal{G}^\mathcal{S}, \mathcal{X}, g , h) \in \mathcal{Q}$.\label{assumption:heuristic}
\end{enumerate}
\end{assumption}


\begin{thm}[Completeness and Optimality]
Under Assumption~\ref{assumption}, 
SES always finds a deadlock-free TPG $\mathcal{G}$ with $cost(\mathcal{G})=\min\{cost(\mathcal{G}) : \mathcal{G}$ is $\mathcal{G}^\mathcal{S}_{\text{root}}$-producible$\}$.
\end{thm}%
\begin{proof}
    First, SES always terminates within a finite time because there are only finitely many possible operation sequences from $\mathcal{G}^\mathcal{S}_{\text{root}}$ to any TPG, each corresponding to a node that can possibly be added to $\mathcal{Q}$. Second, Theorem~\ref{theorem:existence} ensures that there always exist solutions for SES since $\mathcal{G}^\mathcal{S}_{\text{root}}$ is constructed as in Construction~\ref{construction}. Therefore, to prove the completeness of SES, we just need to prove the following claim: At the beginning of each while-loop iteration, for any deadlock-free $\mathcal{G}^\mathcal{S}_{\text{root}}$-producible TPG $\mathcal{G}$, there exists $\mathcal{G}^\mathcal{S} \in \mathcal{Q}$ such that $\mathcal{G}$ is $\mathcal{G}^\mathcal{S}$-producible.
    Here, we abuse the notation $\mathcal{G}^\mathcal{S} \in \mathcal{Q}$ to denote a node in $\mathcal{Q}$ with STPG $\mathcal{G}^\mathcal{S}$. 
    This claim holds inductively: At the first iteration, $\mathcal{G}^\mathcal{S}_{\text{root}} \in \mathcal{Q}$. During any iteration, if some $\mathcal{G}^\mathcal{S} \in \mathcal{Q}$ such that $\mathcal{G}$ is $\mathcal{G}^\mathcal{S}$-producible is popped on \Cref{line:SES:pop}, then one of the following must hold:%
    \begin{itemize}
        \item $\mathcal{G}^\mathcal{S}$ contains no switchable edge, i.e., $\mathcal{G}^\mathcal{S} = \mathcal{G}$: SES terminates, and the inductive step holds vacuously.
        \item $\mathcal{G}$ is $\mathcal{G}^\mathcal{S}_f$-producible: Since $\mathcal{G}$ is acyclic, so is $red(\mathcal{G}^\mathcal{S}_f)$. Thus, $\mathcal{G}^\mathcal{S}_f$ is added into $\mathcal{Q}$.
        \item $\mathcal{G}$ is $\mathcal{G}^\mathcal{S}_r$-producible: This is symmetric to the above case.
    \end{itemize}%
    In any case, the claim remains true after this iteration. Therefore, SES always outputs a solution within a finite time. 

     Finally, we prove that the output TPG $\mathcal{G}$ has the minimum cost. Assume towards contradiction that when $\mathcal{G}$ is returned, there exists $\mathcal{G}^\mathcal{S}_0 \in \mathcal{Q}$ that can produce a better TPG $\mathcal{G}_\text{better}$ with $cost(\mathcal{G}_\text{better}) < cost(\mathcal{G})$. Yet this is impossible since Lemma \ref{lemma:partial} implies that such $\mathcal{G}^\mathcal{S}_0$ must have a smaller $g+h$ value and thus would be popped from $\mathcal{Q}$ before $\mathcal{G}$. 
\end{proof}

\begin{figure*}
    \centering
    \includegraphics[width=0.94\linewidth]{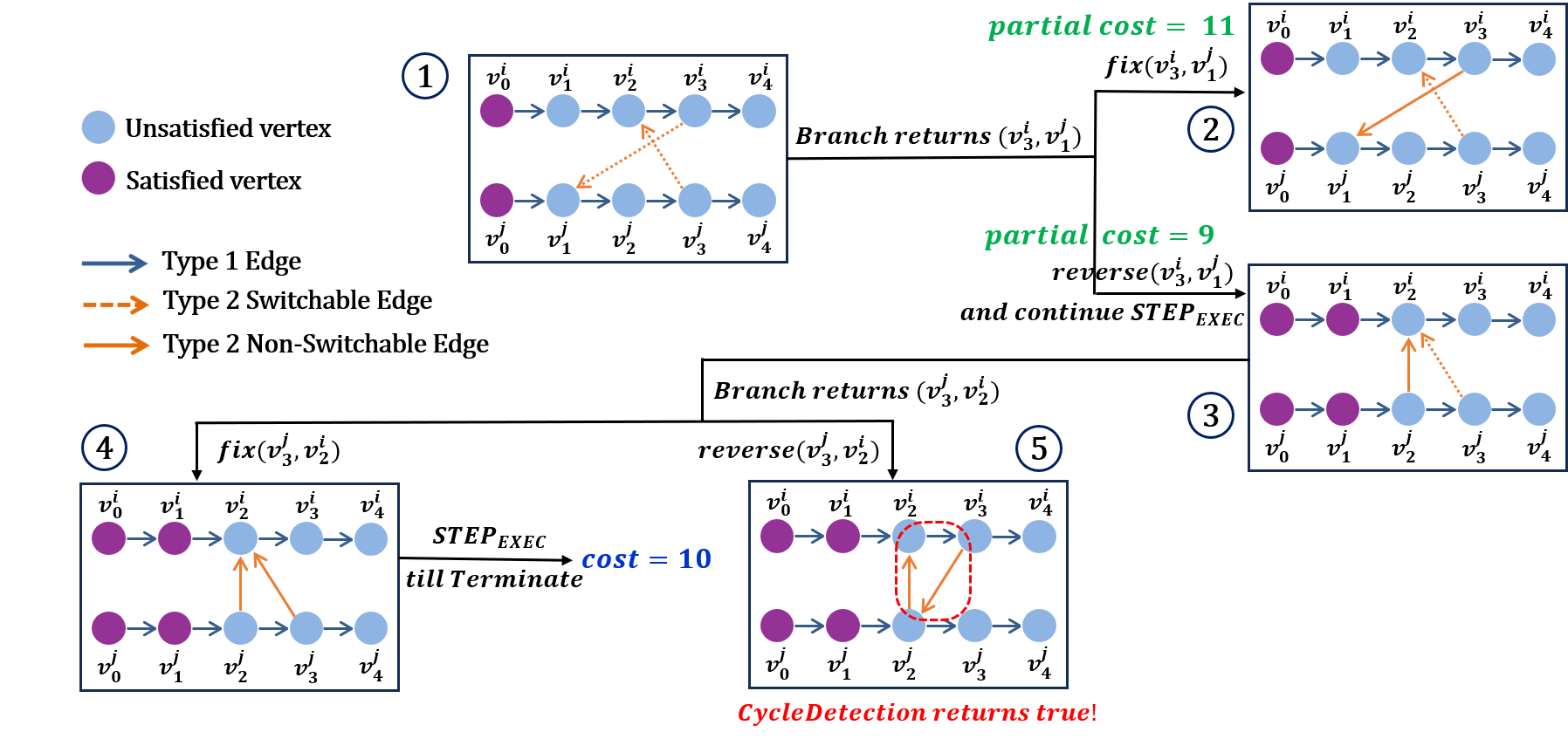}
    \caption{Example of running ESES on the top-left STPG. 
    The circled numbers denote the order of generating these STPGs. 
    }\label{figure:example3}
\end{figure*}
\begin{algorithm}[t!]
\SetAlgorithmName{Module}{algorithmautorefname}{list of algorithms name}
\caption{Execution-based Modules for ESES}\label{module:execution}
    Auxillary information $\mathcal{X}$ is a map $\mathcal{X}: \mathcal{A} \to [0:zi]$, where $\mathcal{X}[i]$ records the index of the most recently satisfied vertex for agent $i$\;

    $\mathcal{X}_\text{init}$ on \Cref{line:SES:root-h} of \Cref{algorithm:framework} maps all agents to $0$;

    $\textsc{Init}_\textsc{exec}$ in Procedure \ref{procedure:execute} is modified by setting  $\mathcal{V}_0$ to
        $\{v^i_k: i \in \mathcal{A}, k \leq \mathcal{X}[i]\}$ instead of $\{v_0^i : i \in \mathcal{A}\}$\;
        
    \Fn{\textsc{Branch}$(\mathcal{G^S}=(\mathcal{V}, \mathcal{E}_1, (\mathcal{S}_{\mathcal{E}2}, \mathcal{N}_{\mathcal{E}2})), \mathcal{X})$}
    {
        \textsc{Init$_{\textsc{EXEC}}$}($\mathcal{V}, \mathcal{X}$)\;\label{line:ESES:call-step}
        $\mathcal{X}' \gets \mathcal{X}$\;
        \While{$\exists v \in \mathcal{V} : v$ \normalfont{unsatisfied}\label{line:ESES:while}}{
            \ForAll{$i \in \mathcal{A} : \mathcal{X}'[i] < zi$\label{line:ESES:select-edge}}
            {
                $v \gets v^i_{\mathcal{X}'[i]+1}$\tcp*{First unsatisfied vertex}
                \If{$\exists e \in \mathcal{S}_{\mathcal{E}2}$ with  $e=(v, u)$ or $(u, v)$\label{line:branch:switchable}}{ 
                    \Return $(\mathcal{X}', cost, e)$\;\label{line:ESES:select-edge-end}
                }
            }
            $\mathcal{S} \gets$ \textsc{Step$_{\textsc{exec}}$}$(red(\mathcal{G^S}))$\;       
            \ForAll{$v^i_{k} \in \mathcal{S}$\label{line:branch:for-s}}{
                Mark $v^i_{k}$ as satisfied\;\label{line:branch:mark-satisfied}
                $\mathcal{X}'[i]\gets k$\;
            }
        }
    \Return $(\mathcal{X}', cost, \textsc{Null})$\;\label{line:branch:return}   
    }
    \Fn{\textsc{Heuristic}$(\mathcal{G^S}, \mathcal{X})$}
    {    
    \Return $(\textsc{Exec}(red(\mathcal{G^S})), \mathcal{X})$\;
    
    }
\end{algorithm}

\section{Execution-based Modules}

In this and the next sections, we describe two sets of modules and prove that they satisfy Assumption \ref{assumption}. We start with describing a set of ``execution-based'' modules in Module \ref{module:execution} and refer to it as \emph{Execution-based SES} (ESES).

In essence, ESES simulates the execution of the STPG and branches when encountering a switchable edge. It uses $\mathcal{X}$ to record the index of the most recently satisfied vertex for every agent, indicating their current locations. $\mathcal{X}$ is updated by the \textsc{Branch} module, which largely ensembles \textsc{Exec} in Procedure \ref{procedure:execute}. At the beginning of each while-loop iteration of \textsc{Branch}, ESES first checks whether the next vertex of any agent is involved in a switchable edge and, if so, returns that edge together with the updated $\mathcal{X}'$ and the cost of moving agents from the old $\mathcal{X}$ to the new $\mathcal{X}'$ [\Crefrange{line:ESES:select-edge}{line:ESES:select-edge-end}], where the cost is updated inside function $\textsc{Step$_{\textsc{exec}}$}$. If no such edge is found, it runs \textsc{Step$_{\textsc{exec}}$} on the reduced TPG to move agents forward by one timestep and repeat the process. 

\begin{example}
\Cref{figure:example3} shows an example of ESES.\footnote{We note that \Cref{figure:example3} also works as an example for the GSES implementation in the following section. The only difference is that GSES does not use the notion of ``(un)satisfied vertex'' or $\textsc{step}_\textsc{exec}$.} 
We start with the top-left STPG \textcircled{1} containing two switchable edges. ESES looks at a ``horizon'' containing the first unsatisfied vertices $v^i_1$ and $v^j_1$, and then picks the adjacent switchable edge $(v^i_3, v^j_1)$ to branch on. This leads to two copies of STPGs \textcircled{2} and \textcircled{3}, containing non-switchable edge $(v^i_3, v^j_1)$ or $(v^j_2, v^i_2)$, respectively. ESES expands on STPG \textcircled{3} first as it has a smaller $g+h$ value. The next switchable edge it encounters is $(v^j_3, v^i_2)$. ESES first $fix$es it and generates TPG \textcircled{4} with $cost = 10$, which is the optimal solution. When ESES $reverse$s the edge, the resulting TPG \textcircled{5} is pruned as it contains a cycle. Note that STPG \textcircled{2} will not be expanded since it has a partial cost greater than the cost of TPG \textcircled{4}. 
\end{example}

\begin{prop}\label{proposition:satisfy1}
    Module \ref{module:execution} satisfies Assumption \ref{assumption}.
\end{prop}
\begin{proof}
    Assumption~\ref{assumption:branch} holds by design. 
    To prove Assumption~\ref{assumption:heuristic}, we first prove the following claim by induction: for every node $(\mathcal{G^S}, \mathcal{X}, g, h) \in \mathcal{Q}$,
    $g$ is the cost of moving agents from their start locations to $\mathcal{X}$. This holds for the root node with $g=0$ and $\mathcal{X}[i]=0, \forall i \in \mathcal{A}$. When we expand a node $(\mathcal{G}, \mathcal{X}, g , h) \in \mathcal{Q}$, $g_\Delta$ returned by \textsc{Branch} is the cost of moving agents from $\mathcal{X}$ to $\mathcal{X}'$. Thus, on \Cref{line:SES:insert1,line:SES:insert2} of \Cref{algorithm:framework}, the $g$ value of the child nodes are $g+g'$, which is the cost of moving agents from their start locations to $\mathcal{X}'$ ($=\mathcal{X}_f=\mathcal{X}_r$). So our claim holds. Module \textsc{heuristic} runs function \textsc{Exec} to compute the cost of moving agents from $\mathcal{X}'$ to their goal locations on the reduced TPG, making $g+h$ the partial cost of $\mathcal{G^S}$ for every node $(\mathcal{G^S}, \mathcal{X}, g , h) \in \mathcal{Q}$.   
\end{proof}

\section{Graph-based Modules}

We now introduce an alternative set of modules that focus on the graph properties of a TPG. We refer to this implementation as \emph{Graph-based SES} (GSES). We will see later in our experiment that this shift of focus significantly improves the efficiency of SES. We start by presenting the following crucial theorem that provides a graph-based approach to computing the cost of a TPG.

Given a TPG $\mathcal{G}$ and a vertex $v \in \mathcal{V}$, let $lp(v)$ denote the longest path among the longest paths from every vertex $v^i_0, i \in \mathcal{A}$ to vertex $v$ on $\mathcal{G}$ and $|lp(v)|$ denote its length.  

\begin{thm}\label{theorem:longestPath}   
    When we execute a TPG, every vertex $v$ is marked as satisfied in the $|lp(v)|^{\text{th}}$ iteration of the while-loop of \textsc{Exec} in Procedure~\ref{procedure:execute}.
\end{thm}

\begin{proof}
    We induct on iteration $t$ and prove that all vertices $v$ with $|lp(v)|=t$ are marked as satisfied in the $t^{\text{th}}$ iteration. In the base case, $\{v^i_0, i \in \mathcal{A}\}$ are the vertices with $|lp(v)|=0$ and are marked as satisfied in the $0^{\text{th}}$ iteration. In the inductive step, we assume that, by the end of the $(t-1)^{\text{th}}$ iteration, all vertices $v$ with $|lp(v)| < t$ are satisfied, and all vertices $v$ with $|lp(v)| \ge t$ are unsatisfied. Then, in the ${t}^{\text{th}}$ iteration, every vertex $v$ with $|lp(v)| = t$ is marked as satisfied because all of its in-neighbors $v'$ have $|lp(v')| < |lp(v)| = t$ and are thus satisfied. For every vertex $v$ with $|lp(v)| > t$, the vertex right before $v$ on $lp(v)$, denoted as $v'$, has $|lp(v')| = |lp(v)| - 1 \geq t$ and is thus unsatisfied on \Cref{line:exec:call-step}. Thus, every vertex $v$ with $|lp(v)| > t$ has at least one in-neighbor unsatisfied and thus remains unsatisfied. Therefore, the theorem holds.  
\end{proof}

Hence, the last vertex $v_{zi}^i$ of every agent $i \in \mathcal{A}$ is marked as satisfied in the $|lp(v_{zi}^i)|^{\text{th}}$ iteration, namely the travel time of agent $i$ is $|lp(v_{zi}^i)|$. We thus get the following corollary.

\begin{cor}\label{corollary:partial-cost}
    Given a TPG $\mathcal{G}$, $cost(\mathcal{G})=\sum_{i \in \mathcal{A}} |lp(v_{zi}^i)|$. 
\end{cor}

An interesting observation is that, if $lp(u) < lp(v)$ for a given TPG $\mathcal{G}$, then adding edge $(u, v)$ to $\mathcal{G}$ does not change its cost since adding $(u, v)$ does not change any longest paths. We thus get the following corollary that is useful later.

\begin{cor}\label{corollary:fix-edge}
    Given a STPG $\mathcal{G^S}$, we compute $lp(v)$ on $red(\mathcal{G^S})$. For any switchable edge with $lp(u) < lp(v)$, $fix$ing it does not change the partial cost of $\mathcal{G^S}$. 
\end{cor}


We adopt the following well-known algorithm to compute $lp(v)$ on a given deadlock-free TPG $\mathcal{G}=(\mathcal{V}, \mathcal{E}_1, \mathcal{E}_2)$:
(1) Set $lp(v) =0, \forall v \in \mathcal{V}$.
(2) Compute a topological sort of all vertices in $\mathcal{V}$.
(3) For each vertex $v$ in the topological order, we set $lp(u) = \max\{lp(u), lp(v) + 1\}$ for every out-neighbor $u$ (namely $(v, u) \in \mathcal{E}_1 \cup \mathcal{E}_2$).  
The time complexity of this longest-path algorithm is $\mathcal{O}(\lvert \mathcal{V} \rvert + \lvert \mathcal{E}_1 \cup \mathcal{E}_2 \rvert)$. 

\begin{algorithm}[t!]
 \SetAlgorithmName{Module}{algorithmautorefname}{list of algorithms name}
 \caption{Graph-based Modules for GSES}
 \label{module:graph}

Auxillary information $\mathcal{X}$ is a map $\mathcal{X}: \mathcal{V} \to [0, \lvert \mathcal{V} \rvert)$, which records $lp(v)$ for every vertex $v$\;

$\mathcal{X}_\text{init}$ on \Cref{line:SES:root-h} of \Cref{algorithm:framework} is empty\;


\Fn{\textsc{Branch}$(\mathcal{G^S}=(\mathcal{V}, \mathcal{E}_1, (\mathcal{S}_{\mathcal{E}2}, \mathcal{N}_{\mathcal{E}2})), \mathcal{X})$}
{
    \If{$\exists (u, v) \in \mathcal{S}_{\mathcal{E}2} :\mathcal{X}[u] \geq \mathcal{X}[v]$}{
        \Return ($\mathcal{X}, 0, (u, v))$\;
    }
    $fix$ all swtichable edges in $\mathcal{G^S}$\;
    \Return ($\mathcal{X}$, 0, \textsc{Null})\;
}


\Fn{\textsc{Heuristic}$(\mathcal{G^S}, \mathcal{X})$}
{
$\mathcal{X}' \gets$ $lp$-values of all vertices in $\mathcal{V}$ on $red(\mathcal{G^S})$\;

\Return $(\sum_{i \in \mathcal{A}} \mathcal{X}[v^i_{zi}], \mathcal{X}')$\;

}
\end{algorithm}

With this algorithm, we specify the graph-based modules in Module~\ref{module:graph}.
In GSES, $\mathcal{X}$ records $lp(v)$ for every vertex $v \in \mathcal{V}$ and is updated by the \textsc{Heuristic} module. Since, with $\mathcal{X}$, \textsc{Heuristic} can directly compute the partial cost of a given STPG, GSES does not use any $g$ values.
The $\textsc{Branch}$ module chooses a switchable edge $(u, v)$ with $\mathcal{X}[u] \geq \mathcal{X}[v]$ to branch on. If no such edge exists, then, by Corollary~\ref{corollary:fix-edge}, $fix$ing all switchable edges produces a TPG with the same cost as the current partial cost. Thus, GSES $fix$ all such edges and terminates in this case.

\begin{prop}\label{proposition:satisfy2}
    Module \ref{module:graph} satisfies Assumption \ref{assumption}.
\end{prop}
\begin{proof}
    Assumption~\ref{assumption:branch} holds by design. Assumption~\ref{assumption:heuristic} holds because of Corollary~\ref{corollary:partial-cost}.
\end{proof}

\begin{figure*}[h]
    \centering
    \includegraphics[width=0.4\linewidth]{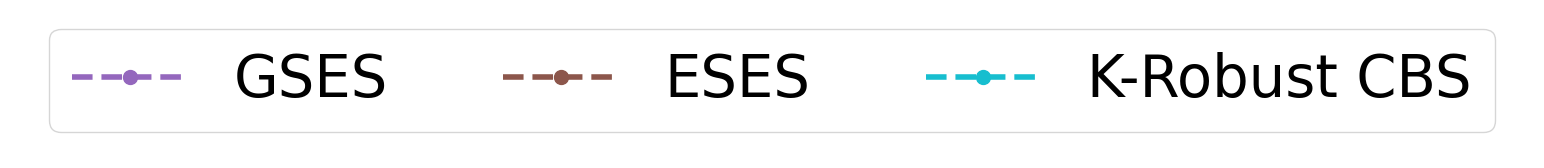}
\includegraphics[width=.93\linewidth]{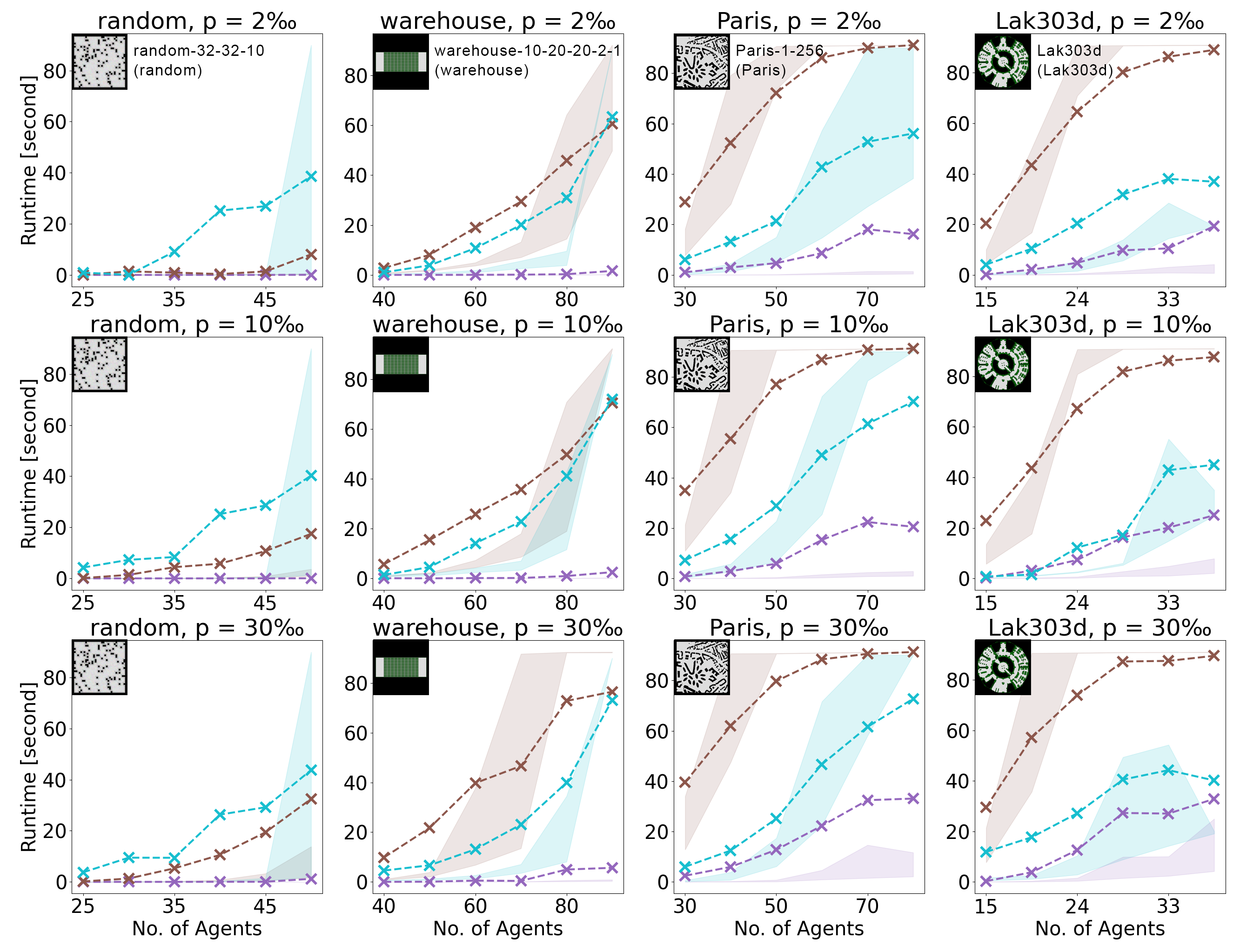}
    \caption{Runtime of ESES, GSES, and k-robust CBS. The dashed lines represent the mean of runtime, and the shaded areas denote the 0.4 to 0.6 quantile range. For trials that exceed the $90$-second time limit, we count it as $90$ seconds.}
    \label{figure:main}
\end{figure*}

\begin{remark}\label{remark:term}
    ESES terminates when all vertices are satisfied, which is possible only when all Type 2 edges are non-switchable. This means that ESES has to expand on all switchable edges before getting a solution. In contrast, GSES can have an \emph{early termination} when $fix$ing all switchable edges does not change any longest paths.
\end{remark}

\begin{figure}[th!]
    \centering
    \includegraphics[width=\linewidth]{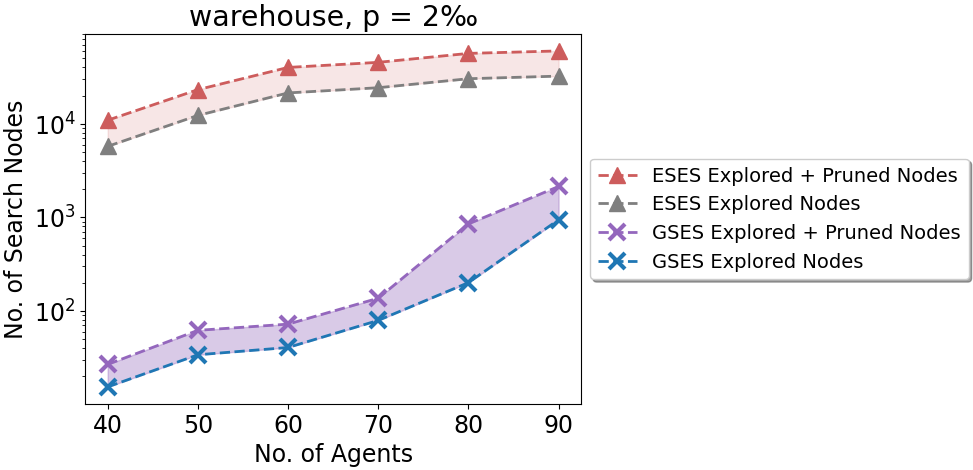}%
    \caption{Numbers of nodes explored and pruned by ESES and GSES on the warehouse map. Dashed lines represent the mean values. Shaded area between two lines for the same algorithm indicates the portion of pruned nodes.}\label{figure:compare}%
\end{figure}

\begin{figure*}
    \centering
    \includegraphics[width=\linewidth]{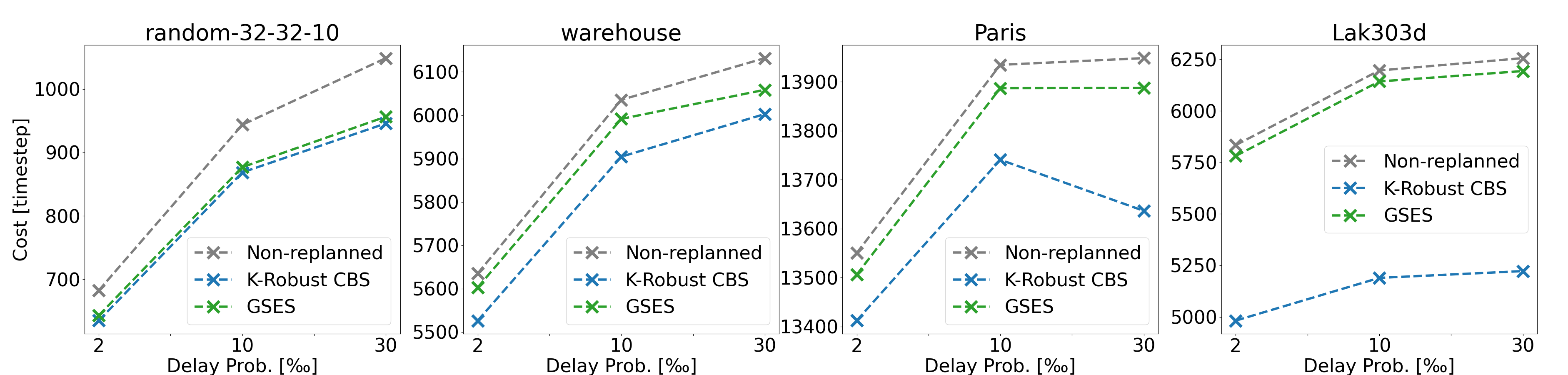}%
    \caption{Mean costs (from the locations where the delays happened to the goal locations) of the non-replanned, GSES-replanned, and K-Robust-CBS-replanned solutions. The means are taken across all trials for all different numbers of agents. }\label{figure:improve}%
\end{figure*}

\section{Experiment}

We use $4$ maps from the MAPF benchmark suite \cite{SternSOCS19}, shown in \Cref{figure:main}, with 6 agent group sizes per map. 
Regarding each map and group size configuration, we run the algorithms on 25 different, evenly distributed scenarios (start/goal locations) with 6 trials per scenario. We set a runtime limit of 90 seconds for each trial. In each trial, we execute the TPG constructed from an optimal MAPF solution planned by a k-Robust MAPF solver k-Robust CBS~\cite{k-robust-sym} with $k=1$.
At each timestep of the execution, each agent that has not reached its goal location is subject to a constant probability $p$ of delay. When a delay happens, we draw a delay length $\Delta$ uniformly random from a range $[10, 20]$, construct a STPG as in Construction \ref{construction}, and run our replanning algorithms. We also develop a baseline that uses k-Robust CBS to find the new optimal solution (that takes into account the delay length $\Delta$) when the delay happens.

We implement all algorithms in C++\footnote{Our SES code is available at
\url{https://github.com/YinggggFeng/Switchable-Edge-Search}. The modified k-Robust CBS code that considers delays is available at \url{https://github.com/nobodyczcz/Lazy-Train-and-K-CBS/tree/wait-on-start}.} and run experiments on a server with a 64-core AMD Ryzen Threadripper 3990X and 192 GB of RAM.

\paragraph{Efficiency}
\Cref{figure:main} compares the runtime of ESES and GSES with replanning using k-Robust CBS. 
In all cases, GSES runs the fastest. Most remarkably, on the random and warehouse maps, the runtime of GSES is consistently below 1 second and does not increase significantly when the number of agents increases, suggesting the potential of GSES for real-time replanning applications.

\paragraph{Comparing ESES and GSES}
We observe from \Cref{figure:main} that, although ESES and GSES adopt the same framework, GSES runs significantly faster than ESES. This is because the longest paths used in GSES defines a simple but extremely powerful early termination condition (see Remark \ref{remark:term}), which enables GSES to find an optimal solution after a very small number of node expansions. \Cref{figure:compare} compares the number of search nodes of ESES and GSES, where \emph{explored nodes} are nodes popped from the priority queue, and \emph{pruned nodes} are nodes pruned by $\textsc{CycleDetection}$. The gap between the red and grey lines (and the gap between the purple and blue lines) indicates the effectiveness of cycle detection for pruning unnecessary search nodes. The gap between the grey and blue lines indicates the effectiveness of early termination as described in Remark \ref{remark:term}.

\paragraph{Improvement of Solution Cost}
\Cref{figure:improve} measures the cost of our replanned solution, in comparison to the non-replanned solution produced by the original TPG and the replanned solution produced by k-Robust CBS. We stress that our solution is guaranteed to be optimal, as proven in previous sections, upon sticking to the original location-wise paths, while k-Robust CBS finds an optimal solution that is independent of the original paths. Therefore, the two algorithms solve intrinsically different problems, and the results here serve primarily for a quantitative understanding of how much improvement we can get by changing only the passing orders of agents at different locations. \Cref{figure:improve} shows that the cost improvement depends heavily on the maps. For example, our solutions have costs similar to the globally optimal solutions on the random map, while the difference is larger on the Lak303d map.

\section{Conclusion}

We proposed Switchable Edge Search to find the optimal passing orders for agents that are planned to visit the same location. We developed two implementations based on either execution (ESES) or graph (GESE) presentations. On the random and warehouse maps, the average runtime of GSES is faster than 1 second for various numbers of agents. On harder maps (Paris and game maps), it also runs faster than replanning with a k-Robust CBS algorithm.


\appendix

\newtheorem{innercustomgeneric}{\customgenericname}
\providecommand{\customgenericname}{}
\newcommand{\newcustomtheorem}[2]{%
  \newenvironment{#1}[1]
  {%
   \renewcommand\customgenericname{#2}%
   \renewcommand\theinnercustomgeneric{##1}%
   \innercustomgeneric
  }
  {\endinnercustomgeneric}
}

\newcustomtheorem{customthm}{Theorem}
\newcustomtheorem{customlemma}{Lemma}
\newcustomtheorem{customprop}{Proposition}
\newcustomtheorem{customcor}{Corollary}

\section{Appendix: Proofs for Section 3}

We rely on the following lemma to prove Proposition~\ref{proposition:cost}.

\begin{customlemma}{A}\label{lemma:plan}
        For every tuple $(l_{k}^i, t^i_k)$ in every path $p_i \in \mathcal{P}$, the corresponding vertex $v^i_k$ in $\mathcal{G}$ is satisfied after the $(t^i_k)^{\text{th}}$ iteration of the while-loop of \textsc{Exec} in Procedure~\ref{procedure:execute}. 
\end{customlemma}%
\begin{proof}
    We induct on the while-loop iteration $t$ and prove that any vertex $v^i_k$ with $t^i_k \le t$ is satisfied after iteration $t$. When $t = 0$, this holds because of $\textsc{init}_\textsc{exec}$. Assume that all vertices $v^i_k$ with $t^i_k \le t-1$ are satisfied after iteration $t-1$. At iteration $t$, for Type 1 edge $(v^i_{k-1}, v^i_k)$, $v^i_{k-1}$ is satisfied as $t^i_{k-1} < t^i_{k} \leq t$. For any Type 2 edge $(v^j_s, v^i_k)$, $v^j_s$ is also satisfied as $t^j_{s} < t^i_k \leq t$ by Definition~\ref{definition:TPG}. This shows that all in-neighbors of $v^i_k$ are satisfied after iteration $t-1$, thus $v^i_k$ is satisfied after iteration $t$.
\end{proof}
    
\begin{customprop}{1}[Cost]
    $cost(\mathcal{G}) \le cost(\mathcal{P})$.
\end{customprop}%
\begin{proof}
    Lemma \ref{lemma:plan} implies that the last vertex $v_{zi}^i$ of every agent $i$ is satisfied after the $(t^i_{zi})^{\text{th}}$ iteration, i.e., when executing $\mathcal{G}$, the travel time of every agent $i$ is no greater than $t^i_{zi}$. Thus, the proposition holds.
\end{proof}

Proposition~\ref{proposition:collision} is similar to Lemma 4 in \cite{honig2016} with different terms. We include a proof for completeness.

 \begin{customprop}{2}[Collision-Free]
    Executing a TPG does not lead to collisions among agents.
\end{customprop}
\begin{proof}
    According to Procedure~\ref{procedure:execute}, we need to show that, when a vertex $v_{k}^i$ is marked as satisfied [\Cref{line:exec:mark-satisfied}], moving agent $i$ to its $k^{th}$ location $l^i_k$ does not lead to collisions.
    Assume towards contradiction that agent $i$ indeed collides with another agent $j$ at timestep $t$. Let $v^i_k$ and $v^j_s$ be the latest satisfied vertices for agents $i$ and $j$, respectively, after the $t^{th}$ iteration of the while-loop. 
    If $i$ and $j$ collide because they are at the same location, then $l^i_k = l^j_s$, indicating that either edge $(v^i_{k+1}, v^j_s)$ or edge $(v^j_{s+1}, v^i_k)$ is in $\mathcal{E}_2$. But this is impossible as neither $v^i_{k+1}$ nor $v^j_{s+1}$ is satisfied. 
    
    If they collide because $j$ leaves a location at timestep $t$, and $i$ enters the same location at timestep $t$, then $v^i_{k}$ and $v^j_{s}$ are marked as satisfied exactly at the $t^{th}$ iteration with $l^i_{k} = l^j_{s-1}$, indicating that either $(v^i_{k}, v^j_s)$ or $(v^j_{s+1}, v^i_{k-1})$ is in $\mathcal{E}_2$. But this is also impossible as neither $v^i_{k}$ nor $v^j_{s+1}$ was satisfied before the $t^{th}$ iteration. 
\end{proof}

\begin{customlemma}{3}[Deadlock $\iff$ Cycle]
    Executing a TPG encounters a deadlock iff the TPG contains cycles.
\end{customlemma}

\begin{proof}
 If a TPG $\mathcal{G}$ has a cycle, executing it will encounter a deadlock as no vertices in the cycle can be marked as satisfied. If executing $\mathcal{G}$ encounters a deadlock in the $t^{\text{th}}$ iteration of the while-loop, we prove that $\mathcal{G}$ has a cycle by contradiction. 
 Let $\mathcal{V}'$ denote the set of unsatisfied vertices, which is non-empty by \Cref{definition:deadlock}. 
 If $\mathcal{G}$ is acyclic, then there exists a topological ordering 
 of $\mathcal{V}'$, and $\mathcal{S}$ must contain the first vertex in the topological ordering as all of its in-neighbors are satisfied, contradicting the deadlock condition of $\mathcal{S} = \emptyset$.
 \end{proof}

\begin{customlemma}{4}[Deadlock-Free]   
    If a TPG is constructed from a MAPF solution, then executing it is deadlock-free.
\end{customlemma}
\begin{proof}
    If a deadlock is encountered, then the execution would enter the while-loop for infinitely many iterations, and $cost$ strictly increases in each iteration. Thus, $cost(\mathcal{G}) = \infty$. Yet, $cost(\mathcal{P})$ is finite, contradicting Proposition \ref{proposition:cost}.
\end{proof}



\section{Acknowledgement}
The research at Carnegie Mellon University was supported by the National Science Foundation (NSF) under Grant 2328671.
The views and conclusions contained in this document are those of the authors and should not be interpreted as representing the official policies, either expressed or implied, of the sponsoring organizations, agencies, or the U.S. government.

\bibliography{main}

\end{document}